\documentclass[letterpaper, 10 pt, journal, twoside]{IEEEtran}

\IEEEoverridecommandlockouts                              

\usepackage{amsmath,amsfonts}

\usepackage{amsthm}  
\usepackage{booktabs}
\usepackage{algorithm}
\usepackage{color}
\usepackage{algpseudocode}
\usepackage{algorithmicx}
\usepackage[capitalize]{cleveref}
\usepackage{bm}

\newtheoremstyle{runindef}
  {}   
  {}   
  {\normalfont}      
  {}                  
  {\itshape}          
  {:}                 
  { }                  
  {\thmname{#1}\ \thmnumber{#2}\thmnote{\,(#3)}}

\theoremstyle{runindef}
\newtheorem{definition}{Definition}

\newtheoremstyle{runinthm}
  { }               
  { }               
  {\normalfont}     
  { }               
  {\bfseries}       
  {:}               
  { }               
  {\thmname{#1}\ \thmnumber{#2}\thmnote{\,(#3)}}

\theoremstyle{runinthm}
\newtheorem{theorem}{Theorem}
\newtheorem{lemma}{Lemma}
 
\usepackage{array}
\usepackage{mathtools}
\usepackage{textcomp}
\usepackage{stfloats}
\usepackage{url}
\usepackage{verbatim}
\usepackage{graphicx}
\usepackage{cite}
\usepackage{siunitx}
\usepackage{libgreek}

\usepackage[font=small]{caption}
\usepackage[font=footnotesize]{subfig}

\usepackage{acro}
\newcommand{\newac}[2]{\DeclareAcronym{#1}{short=#1,long=#2}}
\newac{APF}{Artificial Potential Field}
\newac{CBF}{Control Barrier Function}
\newac{CLF}{Control Lyapunov Function}
\newac{CC}{Constant Curvature}
\newac{CS}{Constant Strain}
\newac{COM}{Center of Mass}
\newac{DCSAT}{Differentiable Conservative SAT}
\newac{DCM}{Discretized Cosserat Rod Model}
\newac{DOF}{Degrees of Freedom}
\newac{EOM}{Equations of Motion}
\newac{GJK}{Gilbert-Johnson-Keerthi}
\newac{GVS}{Geometric Variable Strain} 
\newac{HOCBF}{High-Order CBF}
\newac{HOCLF}{High-Order CLF}
\newac{JIT}{Just-in-Time}
\newac{LNN}{Lagrangian Neural Network}
\newac{LSE}{LogSumExp}
\newac{LSTM}{Long Short-Term Memory}
\newac{MAE}{Mean Absolute Error}
\newac{ML}{Machine Learning}
\newac{MSE}{Mean Squared Error}
\newac{MPC}{Model Predictive Control}
\newac{OLSAT}{One-Log SAT}
\newac{PAC}{Piecewise Affine Curvature}
\newac{PCS}{Piecewise Constant Strain}
\newac{PCC}{Piecewise Constant Curvature}
\newac{PDE}{Partial Differential Equation}
\newac{QP}{Quadratic Program}
\newac{RL}{Reinforcement Learning}
\newac{RMSE}{Root Mean-Squared Error}
\newac{SAT}{Separating Axis Theorem}
\newac{SSAT}{Smooth SAT}

\begin{document}
\bstctlcite{IEEEexample:BSTcontrol}

\title{Contact-Aware Safety in Soft Robots Using High-Order Control Barrier and Lyapunov Functions }


\author{
Kiwan Wong$^{1,2}$, Maximilian St\"olzle$^{1,2,3}$, Wei Xiao$^{4,2}$, Cosimo Della Santina$^3$, Daniela Rus$^{*,2}$, Gioele Zardini$^{*,1}$

\thanks{Manuscript received: May 4, 2025; Revised August 6, 2025; Accepted September 24, 2025.}%
\thanks{This paper was recommended for publication by Editor Cecilia Laschi upon evaluation of the Associate Editor and Reviewers' comments.}%
\thanks{$^*$D.~Rus and G.~Zardini contributed equally as joint senior authors.}%
\thanks{
    $^{1}$Laboratory for Information and Decision Systems, Massachusetts Institute of Technology, Cambridge, MA 02139, USA {\tt\scriptsize \{kiwan588, mstolzle, gzardini\}@mit.edu}.
    $^{2}$Computer Science and Artificial Intelligence Laboratory, Massachusetts Institute of Technology, Cambridge, MA 02139, USA {\tt\scriptsize \{ weixy, rus \}@mit.edu}.
    $^{3}$Cognitive Robotics, Delft University of Technology, Delft, 2628 CD, Netherlands {\tt\scriptsize \{M.W.Stolzle, C.DellaSantina\}@tudelft.nl}.
    $^{4}$Robotics Engineering Department, Worcester Polytechnic Institute, Worcester, MA 01609, USA {\tt\scriptsize wxiao3@wpi.edu}.
}%
\thanks{The work by K.~Wong was supported by The Hong Kong Jockey Club Scholarships; The work by M.~Stölzle was supported under the European Union's Horizon Europe Program from Project EMERGE - Grant Agreement No. 101070918, and by the Cultuurfonds Wetenschapsbeurzen 2024 and the Rudge (1948) and Nancy Allen Chair for his research visit to LIDS/Zardini Lab at MIT.}
\thanks{Digital Object Identifier (DOI): see top of this page.}
}
\markboth{IEEE ROBOTICS AND AUTOMATION LETTERS. PREPRINT VERSION. ACCEPTED September, 2025}
{Wong \MakeLowercase{\textit{et al.}}: Contact-Aware Safety in Soft Robots Using High-Order Control Barrier and Lyapunov Functions}


\maketitle

\begin{abstract}
Robots operating alongside people, particularly in sensitive scenarios such as aiding the elderly with daily tasks or collaborating with workers in manufacturing, must guarantee safety and cultivate user trust.  Continuum soft manipulators promise safety through material compliance, but as designs evolve for greater precision, payload capacity, and speed, and increasingly incorporate rigid elements, their injury risk resurfaces. In this letter, we introduce a comprehensive High-Order Control Barrier Function (HOCBF) + High-Order Control Lyapunov Function (HOCLF) framework that enforces strict contact force limits across the entire soft-robot body during environmental interactions. Our approach combines a differentiable Piecewise Cosserat-Segment (PCS) dynamics model with a convex-polygon distance approximation metric, named Differentiable Conservative Separating Axis Theorem (DCSAT), based on the soft robot geometry to enable real-time, whole-body collision detection, resolution, and enforcement of the safety constraints. By embedding HOCBFs into our optimization routine, we guarantee safety, allowing, for instance, safe navigation in operational space under HOCLF-driven motion objectives. Extensive planar simulations demonstrate that our method maintains safety-bounded contacts while achieving precise shape and task-space regulation. This work thus lays a foundation for the deployment of soft robots in human-centric environments with provable safety and performance.
\end{abstract}


\section{Introduction}
\IEEEPARstart{D}{eploying} robots in human-centered environments, such as assisting workers in manufacturing or supporting older adults in everyday activities~\cite{hall2019acceptance}, demands not only demonstrable safety but also user confidence in the robot's behavior.
Traditional rigid collaborative manipulators address this need through increasingly sophisticated algorithms for collision detection~\cite{haddadin2013towards}, impedance control~\cite{khatib1987unified}, \ac{MPC}~\cite{pupa2024efficient}, and, more recently, \acp{CBF}~\cite{ames2016control, ferraguti2020control}, and the successful integration of Lyapunov-based methods with reinforcement learning to ensure robotic safety~\cite{han2021reinforcement}. 
Yet, perception errors or model inaccuracies can still expose users to hazardous impacts.

Continuum soft manipulators offer a fundamentally different path to safety: instead of relying solely on software, they seek to embed safety directly through compliance~\cite{rus2015design}.
However, material softness is not a panacea~\cite{stolzle2025soft, dickson2025safe}.
As the field advances toward greater precision and functionality, emerging designs are expected to incorporate increased stiffness, exert larger forces and velocities~\cite{haggerty2023control}, and adopt hybrid rigid-soft architectures~\cite{patterson2024safe}.
Such developments reintroduce risks traditionally associated with rigid systems.
Thus, mechanical compliance must be augmented with algorithmic guarantees that ensure real-time safety and foster user trust.

\begin{figure}[tb]
\vspace{-4mm}
  \centering
  \includegraphics[width=1.0\linewidth]{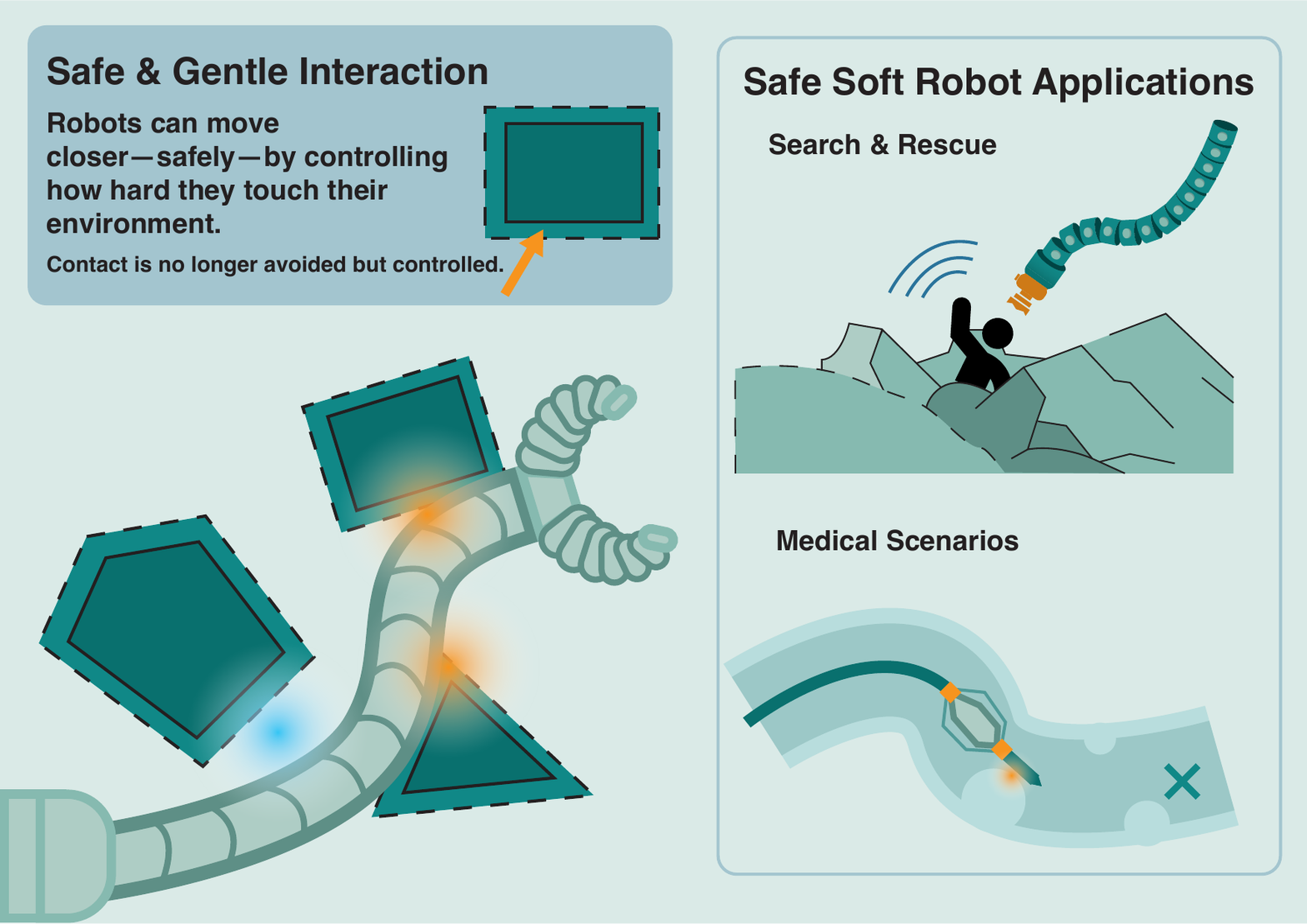}
  \caption{\small
    \textbf{Contact Safety-Aware Control of Soft Robots with \acp{HOCBF} and \acp{HOCLF}.}
    Illustration of compliant contact control with safety bounds guaranteed by \acp{HOCBF}.
    By respecting contact force limits, the robot can intentionally engage with its surroundings without sacrificing safety.
    Task goals are shaped through \acp{HOCLF}, while constraint satisfaction is upheld by \acp{HOCBF}.
    This approach enables the secure use of soft robots in demanding settings—from search‑and‑rescue missions to delicate medical procedures.
  }
  \label{fig:teaser}
\end{figure}

Our approach embraces, rather than avoids, physical contact\footnote{Please note that we use ``contact'' and ``collision'' interchangeably.}~\cite{rao2024towards, xu2024hybrid, dickson2025safe}, deviating from the predominant paradigm in rigid robotics~\cite{haddadin2013towards,iso2016collaborative}.
Instead of treating contact as a failure, we exploit the soft robot’s embodied intelligence~\cite{mengaldo2022concise}—its intrinsic physical coupling with the environment—to enhance robustness, stabilize deformation and motion, and adapt to external constraints, all while ensuring that every interaction respects safety standards such as the injury severity thresholds set out in ISO/TS 15066:2016~\cite{iso2016collaborative}.

To date, no method explicitly enforces upper bounds on the contact force or pressure applied across the entire surface of a soft robot, while accounting for the system's inertia.
Classical impedance and force control schemes cannot impose strict bounds~\cite{della2020model}, and \ac{CBF}-based methods~\cite{patterson2024safe} address only self-contact avoidance.
A recently published paper~\cite{dickson2025safe} adapts \acp{CBF} as a safety filter to constrain end-effector forces, but it neglects distributed body interactions and relies on a simplified template model that approximates the continuum bending behavior with an articulated chain of masses connected by prismatic joints~\cite{della2020model} while neglecting important strains such as shear or elongation.
Instead, Xu \textit{et al.}~\cite{xu2024hybrid} measures contact forces between the soft robot body and obstacles in the environment, but defines a kinematics-aware \ac{CBF} instead of a dynamics-aware \ac{HOCBF}; therefore, it neglects the dynamics of the system, which can, in turn, cause safety issues.

To fill this research gap, we present an integrated control framework that imposes contact‑force limits along the entire body of a soft manipulator based on differentiable dynamic strain models. Building on the well‑established \ac{CBF}+\ac{CLF} framework~\cite{ames2016control}—specifically its high‑order extension~\cite{xiao2021highclbf}—our method optimizes for a control objective encoded in an \ac{HOCLF} while keeping the system’s trajectory inside a certified safe set as specified by \ac{HOCBF} constraints by solving a constrained \ac{QP} online. The proposed \ac{HOCBF}+\ac{HOCLF} control scheme rests on two pillars: (i) a fully differentiable implementation of the \ac{PCS}~\cite{renda2018discrete,stolzle2024experimental}, and (ii) a fast, differentiable collision detection \& resolution routine that represents the soft manipulator with convex‑polygon approximations.

In support of pillar (ii), we propose a new convex‑polygon distance measure—\ac{DCSAT}—that serves as a conservative, differentiable proxy for the standard \ac{SAT} metric. Compared with recent differentiable SAT surrogates such as \ac{SSAT}~\cite{takasugi2024real}, our approach (1) systematically underestimates the separation distance, yielding the conservative buffer required for formal safety guarantees, and (2) achieves an increase of roughly 1.5–3× in computational efficiency, enabling real-time, full-body collision checks. We then validate the proposed framework through extensive simulations in a planar setting.

In summary, this letter (i) develops a principled \ac{HOCBF}-based method to enforce global contact force constraints on soft robots, (ii) adapts the \ac{HOCBF}+\ac{HOCLF} framework for operational space regulation supporting navigation tasks, 
and (iii) presents \ac{DCSAT}, a new, fast, and conservative differentiable collision-detection method for arbitrary convex geometries, improving the feasibility of real-time use for full-body safety guarantees.
A video attachment is available on YouTube\footnote{\url{https://youtu.be/ahUhVXiRDPE}} and the code is open-sourced on GitHub\footnote{\url{https://edu.nl/yekue}}.
\section{Background}
\label{sec:background}
This section introduces the background necessary for introducing the methodology and, later, the baseline methods.
For clarity and simplicity, we focus on the planar case throughout this letter, though the framework naturally generalizes to 3D scenarios.
\subsection{Soft Robotic Kinematics and Dynamics}
We model the soft robot kinematics using the \ac{PCS} formulation~\cite{renda2018discrete}, which approximates the continuous backbone by discretizing it into $N$ segments, where each segment exhibits the spatially constant strain $\bm{\xi}_i \in \mathbb{R}^3$.
The robot configuration $\bm{q}$ is then defined as the deviation of these strains from their equilibrium values, yielding a generalized coordinate vector $\bm{q} \in \mathbb{R}^{3N}$.
Based on this kinematic model, the forward kinematics map~$\bm{\chi} = \operatorname{FK}(\bm{q}, s):  \mathbb{R}^{3N} \times (0, L] \rightarrow SE(2)$ returns the Cartesian position $\bm{\chi} = \begin{bmatrix}
    \theta & p_\mathrm{x} & p_\mathrm{y}
\end{bmatrix}^\top$ at a given arc-length position~$s$ along the backbone, where $\theta \in [-\pi, \pi)$ is the planar orientation, $p_\mathrm{x}, p_\mathrm{y}$ are the x- and y-positions, and $L \in \mathbb{R}$ is the total arc length of the robot's centerline.
Differentiating $\operatorname{FK}(\bm{q}, s)$ with respect to the configuration results in the Jacobian $\bm{J}(\bm{q},s) = \frac{\partial \operatorname{FK}(\bm{q},s)}{\partial \bm{q}} \in \mathbb{R}^{3 \times 3N}$.

The corresponding dynamics can be derived leveraging established multibody modeling procedures~\cite{della2023model}, resulting in the following equations of motion:
\begin{equation}\scriptsize\label{eq:multibody2}
\begin{split}
    \underbrace{\bm{M}(\bm{q}) \ddot{\bm{q}} + \bm{C}\bigl(\bm{q},\dot{\bm{q}}\bigr) \dot{\bm{q}} + \bm{G}(\bm{q})}_{\text{Multibody dynamics}} + \underbrace{\bm{K} \bm{q} + \bm{D} \dot{\bm{q}}}_{\text{Elast. and Diss. Forces}}
    = \underbrace{\bm{A}(\bm{q}) \bm{u}}_{\text{Act. Model}} + \underbrace{\bm{\tau}_\mathrm{c}}_\text{Contact},
\end{split}
\end{equation}
where $\bm{M}(\bm{q}), \bm{C}(\bm{q},\dot{\bm{q}}) \in \mathbb{R}^{3N \times 3N}$ are the mass and Coriolis matrices, respectively, $\bm{G}(\bm{q})$ captures gravitational effects, and $\bm{K}, \bm{D} \in \mathbb{R}^{3N \times 3N}$ are the linear stiffness and damping matrices, respectively.
The actuation matrix $\bm{A}(\bm{q}) \in \mathbb{R}^{3N \times m}$ maps the actuation/control input $\bm{u} \in \mathcal{U} \subset \mathbb{R}^m$ into generalized torques and $\bm{\tau}_\mathrm{c} \in \mathbb{R}^{3N}$ accounts for forces and torques generated by contact with the environment.

%
Furthermore, $\bm{J}_{\omicron,\bm{M}}^{+}(\bm{q}) \in \mathbb{R}^{3N \times \omicron}$ is the dynamically consistent pseudo-inverse of the Jacobian $\bm{J}_\omicron$ that maps into the operational space $\bm{\omicron} \in \mathbb{R}^{\omicron}$~\cite{khatib1987unified, della2020model, stolzle2024guiding, stolzle2025phdthesis}.
\subsection{Model-Based Operational Space Controller}
A model-based setpoint regulator $\bm{u}_\omicron(\bm{q}, \dot{\bm{q}}): \mathbb{R}^{6N} \to \in \mathbb{R}^{3N}$ that drives the system towards the desired operational space reference $\bm{\omicron}^\mathrm{d}$ in exponential time can be designed based on the operational space dynamics~\cite{khatib1987unified, della2020model, stolzle2024guiding, stolzle2025phdthesis}
\begin{equation}\scriptsize\label{eq:model_based_operational_space_controller}
\begin{split}
    \bm{u}_\omicron(\bm{q}, \dot{\bm{q}}) = \bm{J}_\omicron^\top(\bm{q}) \, \bm{f} + \bm{G}(\bm{q}),\\
    \bm{f} = \underbrace{\bm{K}_p \, \bm{e}_\omicron + \bm{K}_i \int \bm{e}_\omicron \, \mathrm{d}t + \bm{K}_d \, \dot{\bm{e}}_\omicron}_\text{Operational Space PID}
    + \underbrace{\bm{J}_{\omicron,\bm{M}}^{+^\top} \, \left( \bm{K} \bm{q} + \bm{D} \dot{\bm{q}} \right)}_\text{Cancel. Elastic \& Diss. Forces},
\end{split}
\end{equation}
where $\bm{f} \in \mathbb{R}^{\omicron}$ is a Cartesian force that serves as the control input in operational-space, $\bm{e}_\omicron(t) = \bm{\omicron}^\mathrm{d}(t) - \bm{\omicron}(t)$ is the regulation error in operational space, and $\bm{K}_\mathrm{p}, \bm{K}_\mathrm{i}, \bm{K}_\mathrm{d} \in \mathbb{R}^{\omicron \times \omicron}$ are the corresponding PID gain matrices.

\subsection{High Order Control Barrier Functions and High Order Control Lyapunov Functions}
The soft robot dynamics can be expressed in control-affine form as an ODE:
\begin{equation}\scriptsize\label{eq:soft_robot_dynamics_control_affine}
    \dot{\bm{x}} = 
    \underbrace{
        \left[ \begin{array}{c} \dot{\bm{q}}\\ \bm{M}^{-1}(\bm{\tau}_\mathrm{c} - \bm{C} \dot{\bm{q}} - \bm{G} + \bm{K} \bm{q} - \bm{D} \dot{\bm{q}}) \end{array} \right]
    }_{f(\bm{x})}
    + 
    \underbrace{
        \left[ \begin{array}{c} \bm{0}_{3N \times m}\\ \bm{M}^{-1} \bm{A}(\bm{q}) \end{array} \, \right]
    }_{g(\bm{x})}\bm{u}
    \\,
\end{equation}
with the locally Lipschitz continuous dynamics functions $f(\bm{x}):\mathbb{R}^n \to \mathbb{R}^n$ and $g:\mathbb{R}^n \to \mathbb{R}^{n\times m}$, where $\bm{x} = \begin{bmatrix}
    \bm{q}^\top & \dot{\bm{q}}^\top
\end{bmatrix}^\top \in \mathbb{R}^n$ with $n = 6N$ is soft robot state with the corresponding time derivative $\dot{\bm{x}} \in \mathbb{R}^n$
%
%
%
%
%
In many soft robotic control tasks, constraints specified in operational space do not yield an explicit dependence on the input~$\bm{u}$ after a single time derivative.
This motivates the use of high-order extensions of \acp{CLF} and \acp{CBF}, where constraints are enforced on higher-order derivatives that expose the control input explicitly \cite{xiao2021highclbf}.
%
%
%
%
Given a function~$b$ with relative degree~$r$, one can recursively build an \ac{HOCBF} and \ac{HOCLF} as follows~\cite{xiao2021highclbf}.

\begin{definition}[High-Order Control Barrier Function~{\cite{xiao2021highclbf}}]
Let $b\colon \mathbb{R}^n\to \mathbb{R}$ be $r$-times differentiable with relative degree~$r$. Define recursively:
\[
\psi_0 := b,\quad \psi_i := L_f\psi_{i-1} + \alpha_i(\psi_{i-1}),\quad i=1,\dots,r{-}1,
\]
where each $\alpha_i$ is class-$\mathcal{K}$. Let $C_i := \{\, \bm{x} \mid \psi_{i-1}(\bm{x}) \ge 0\,\}$. Then $b$ is a \ac{HOCBF} if there exists $\alpha_r \in \mathcal{K}$ such that:
{\small
\begin{align*}
\sup_{\bm{u} \in \mathcal{U}} 
\bigl[ 
  L_f^r b 
  + L_g L_f^{r-1} b\, \bm{u} 
  + O(b) 
  + \alpha_r(\psi_{r-1}) 
\bigr]
\ge 0, \quad
\forall\, \bm{x} \in \bigcap_{i=1}^r C_i.
\end{align*}
}
\label{def:hocbf}
\end{definition}

\begin{definition}[High-Order Control Lyapunov Function]
Let $V:\mathbb{R}^n \to \mathbb{R}$ be a differentiable function of relative degree~$r$, and define:
\[
\phi_0 := V,\quad \phi_i := L_f\phi_{i-1} + \beta_i(\phi_{i-1}),\quad i=1,\dots,r{-}1,
\]
with each $\beta_i \in \mathcal{K}$. Then $V$ is a \ac{HOCLF} if there exists $\beta_r \in \mathcal{K}_\infty$ such that:
{\small
\begin{align*}
\inf_{\bm{u} \in \mathcal{U}} 
\bigl[
  L_f^r V 
  + L_g L_f^{r-1} V\, \bm{u} 
  + O(V) 
  + \beta_r(\phi_{r-1})
\bigr]
\le 0, \quad
\forall\, \bm{x} \ne \bm{0}_n.
\end{align*}
}
\label{def:hoclf}
\end{definition}
Here, $L_f$, $L_g$ denote Lie derivatives along $f$, $g$ respectively, and $O(\cdot)$ collects all lower-order Lie and time derivatives up to degree~$r{-}1$.




The \ac{HOCBF} \eqref{def:hocbf}, and the \ac{HOCLF} \eqref{def:hoclf} can be integrated into a \ac{QP} convex optimization problem:
\begin{equation}\footnotesize\label{eq:QP}
\begin{aligned}
    \min_{\bm{u},\delta}\;& \|\bm{u}\|_2^2 + p\,\delta^2, \\
    \text{s.t.}\;&
      L^r_fb(\bm{x})
      + L_g L_f^{r-1}b(\bm{x}) \, \bm{u}
      + O\bigl(b(\bm{x})\bigr) + \alpha_r(\psi_{r-1}(\bm{x}))
      \;\ge\;0,\\[4pt]
    &\;
      L^r_fV(\bm{x})
      + L_g L_f^{r-1}V(\bm{x}) \, \bm{u}
      + O(V(\bm{x})) + \beta_r(\phi_{r-1}(\bm{x}))
      \;\le\;\delta.
\end{aligned}
\end{equation}
To keep the \ac{QP}  when several \acp{HOCBF} and \acp{HOCLF} contradict each other, we add a non‑negative slack $\delta \ge 0$ with penalty $p > 0$. This slack relaxes safety constraints when strict enforcement is impossible due to the nominal input or mutual conflicts. Typically, \acp{HOCLF} capture performance objectives and \acp{HOCBF} safety and other constraints; assigning slack to lower‑priority terms lets the controller trade performance for safety.

\section{High-Order Control Barrier and Lyapunov Function for Environment-Aware Control}
\label{sec:CLFCBFforSORO}
This section describes how we can design \acp{HOCBF} and \acp{HOCLF} for environment-aware soft robot control. 
Most importantly, we can ensure contact force limits, and with this, safety, by deploying \acp{HOCBF} while relying on \acp{HOCLF} for defining motion behavior and objectives.
To enable this, we require access to differentiable algorithms that perform collision detection and resolution between the soft robot body and convex polygonal environment obstacles.
%
We consider a planar soft robotic arm operating within a two‐dimensional workspace $\mathcal{W}\subset\mathbb{R}^2$, populated by~$n_\mathrm{obs}$ known convex obstacles $\mathcal{W}_{\mathrm{obs}}=\{\mathcal{O}_1,\dots,\mathcal{O}_{n_\mathrm{obs}}\}$.
We also assume the robot is fully actuated, with $m = 3N = \frac{n}{2}$.
\subsection{Collision Detection}
\label{sec:segmentation}
To facilitate collision detection and spatial reasoning, the soft robotic arm, though continuously deformable, is approximated as a discrete chain of convex polygonal parts, as illustrated in~\cref{fig:sat}.
Specifically, the robot is segmented into $N_\mathrm{srpoly}$ convex polygons $\mathcal{R}=(P_1,\dots,P_{N_\mathrm{srpoly}})$, each defined by its vertices $\{\mathbf{v}_{i,1},\dots,\mathbf{v}_{i,k_i}\}$, where $k_i$ is the number of vertices of the $i$th part. 
These vertices are computed via the forward kinematics $\operatorname{FK}(\bm{q},s)$ defined earlier. Each polygon is represented as:
\begin{equation}\small
    P_i = \Bigl\{\mathbf{x}\in\mathbb{R}^2 \mid \mathbf{x} = \sum_{j=1}^{k_i}\alpha_{i,j}\mathbf{v}_{i,j},\ \alpha_{i,j}\ge0,\ \sum_{j=1}^{k_i}\alpha_{i,j}=1\Bigr\}.
\end{equation}
The entire robot is then~$\mathcal{R} = (P_i)_{i=1}^{N_\mathrm{srpoly}}$.

This definition now allows us to detect collisions between the $N_\mathrm{srpoly}$ polygons approximating the robot body and the environment approximated by $N_\mathrm{obs}$ static convex polygons.
In order to do so, we need to know the configuration-dependent distance $h_{i,j}(\bm{q}): \mathbb{R}^{3N} \to \mathbb{R}$ between the $i$th soft robot's part $P_i$ and the $j$th obstacle $\mathcal{O}_j$.
Here, a positive $h_{i,j}(\bm{q})$ indicates separation, while a negative value indicates penetration. 
This distance is provided by a polygon distance metric $d(\cdot, \cdot)$ s.t.
\begin{equation*}\small
    h_{i,j}(\bm{q}) = d(P_i(\bm{q}),\mathcal{O}_j),\: i=1,\dots,N_\mathrm{srpoly},\ j=1,\dots,n_\mathrm{obs},
\end{equation*}
between the $i$th soft robot's part $P_i$ and the $j$th obstacle $\mathcal{O}_j$.
For notational simplicity, we omit the explicit dependence on the configuration variable $\bm{q}$ and write $h_{i,j}$ in the remainder of this section.
In principle, any differentiable distance function $d(\cdot,\cdot)$ can be used. In \cref{sec:onelog}, we present \ac{DCSAT}, which is a differentiable, computationally efficient, and conservative version of the \ac{SAT} algorithm.

\subsection{Collision Resolution}
The distance metric $h(\bm{q})$ now allows us to resolve the collision, which means that we project the collision forces onto the soft robot dynamics and vice versa.
In this work, we specifically, without loss of generality, use a linear spring-damper model to capture the collision characteristics. Future work might explore more advanced contact models here.
The collision force $F_\mathrm{c} \in \mathbb{R}_{\geq 0}$ is given as
{\small
\begin{equation*}\footnotesize
    F_{\mathrm{c}_{i,j}}(h_{i,j}, \dot{h}_{i,j}) = \begin{cases}
        0, & \text{if } h_{i,j} > 0 \\[6pt]
        -k_\mathrm{c}\, h_{i,j} - c_\mathrm{c} \, \dot{h}_{i,j}, & \text{if } h_{i,j} \le 0
    \end{cases}
\end{equation*}
}
where~$k_\mathrm{c}$ the contact stiffness, and $c_\mathrm{c}$ the damping coefficient. 
To recover differentiability, we approximate the collision dynamics using softplus-based smoothing with $\varepsilon \in \mathbb{R}_{\geq0}$ as
\begin{equation}\small\label{eq:collision_force_smoothed}
    F_{\mathrm{c}_{i,j}}(h_{i,j}, \dot{h}_{i,j}) =
    k_\mathrm{c} \ln\left(1 + e^{-h_{i,j}/\varepsilon}\right)
    + c_\mathrm{c} \ln\left(1 + e^{-\dot{h}_{i,j}/\varepsilon}\right).
\end{equation}
This contact force can now easily be applied to both the environment and the soft robot by projecting it along the contact surface vector. Specifically, for a given surface normal $\bm{n}_{\mathrm{c}_{i,j}} \in \mathbb{R}^2$ pointing from the obstacle to the soft robot body with $\lVert \bm{n}_{\mathrm{c}_{i,j}} \rVert_2 = 1$ and a contact position $\bm{p}_{\mathrm{c}_{i,j}} \in \mathbb{R}^2$, the generalized contact torque onto the soft robot can be described by
\begin{equation}\small
\begin{split}
    \bm{\tau}_\mathrm{c} = \sum_{i=1}^{N_\mathrm{srpoly}} \sum_{j = 1}^{N_\mathrm{obs}}\bm{J}_{\mathrm{c}_{i,j}}^\top(\bm{q}) \, F_{\mathrm{c}_{i,j}}(h_{i,j}, \dot{h}_{i,j}) \, \bm{n}_\mathrm{c} \in \mathbb{R}^{3N},
\end{split}
\end{equation}
where $\bm{J}_{\mathrm{c}_{i,j}} \in \mathbb{R}^{2 \times 3N}$ is the positional Jacobian of the contact point on the surface of the soft robot that is:
\begin{equation}\footnotesize
    \dot{\bm{p}}_\mathrm{c} = \underbrace{\left ( \bm{J}_{x,y} + \mathrm{diag}(-1, 1) \left ( \bm{p}_{\mathrm{c}_{i,j}} - \operatorname{FK}_{x,y}(\bm{q},s_{\mathrm{c}_{i,j}})  \right )^\top \bm{J}_{\theta} \right )}_{\bm{J}_{\mathrm{c}_{i,j}}(\bm{q})} \dot{\bm{q}}
\end{equation}
with $s_{\mathrm{c}_{i,j}} \in (0, L]$ the point on the backbone of the robot closest to the contact position $\bm{p}_{\mathrm{c}_{i,j}}$ and $\bm{J}_{\theta}(\bm{q},s_{\mathrm{c}_{i,j}})  \in \mathbb{R}^{1 \times 3N}$ and $\bm{J}_{x,y}(\bm{q},s_{\mathrm{c}_{i,j}}) \in \mathbb{R}^{2 \times 3N}$ the orientation and positional rows of the forward kinematics Jacobian $\bm{J}(\bm{q},s_{\mathrm{c}_{i,j}})$, respectively.

\subsection{Ensuring Safety via High-Order CBFs}
Below, we introduce a set of relative-degree-two \acp{HOCBF} dedicated to maintaining safety. We start with an \acp{HOCBF} that is standard in rigid-robotics applications and ensures complete avoidance of contact between the robot and its environment. While effective, these constraints can sharply curb performance, encourage overly cautious behavior, and block truly collaborative human–robot interactions.

To overcome such drawbacks, we highlight a \ac{HOCBF} limiting the contact force to $F_{\mathrm{c},\mathrm{max}} \in \mathbb{R}_{>0}$. This approach allows controlled contact between the soft robot and its surroundings while guaranteeing that such contact remains safe~\cite{dickson2025safe}. Our formulation is inspired by extensive studies on injury-severity criteria for rigid collaborative robots~\cite{haddadin2013towards} and by ISO/TS 15066:2016~\cite{iso2016collaborative}, which specifies body-part–dependent force thresholds as proxies for injury risk.

\begin{enumerate}
    \item \textbf{Contact Avoidance \ac{HOCBF}.} 
    To ensure safety with respect to a forbidden region $\mathcal{A} \subset \mathcal{W}$, we define a \ac{HOCBF} based on the signed distance between the robot segment $P_i$ and $\mathcal{A}$. Let $r_\mathrm{safe} \ge 0$ be a prescribed safety margin. Using the previously defined smooth distance metric $h_{i,\mathcal{A}}(\bm{q}) = d(P_i(\bm{q}), \mathcal{A})$, we define
    \begin{equation}
        b_{i,\mathcal{A}}(\bm{x}) = h_{i,\mathcal{A}}(\bm{q})^2 - r_\mathrm{safe}^2.
    \end{equation}
    Then $b_{i,\mathcal{A}}(\bm{x}) \ge 0$ guarantees that $P_i$ maintains a distance of at least $r_\mathrm{safe} \geq 0$ from $\mathcal{A}$.

    \item \textbf{Contact Force Limit \ac{HOCBF}.}
    For each obstacle $\mathcal{O}_j \in \mathcal{W}_{\mathrm{obs}}$, define
    \begin{equation}
        b_{i,j}(\bm{x}) = F_{\mathrm{c},\mathrm{max},j} - F_{\mathrm{c}_{i,j}}(h_{i,j},\dot{h}_{i,j}),
    \end{equation}
    where $F_{\mathrm{c}_{i,j}}(h_{i,j},\dot{h}_{i,j}) \geq 0$ is the contact force, for example, stemming from a linear spring-damper contact model as defined in \eqref{eq:collision_force_smoothed}, and $F_{\mathrm{c},\mathrm{max},j}$ is maximum allowable contact force in static settings as, e.g., defined in ISO/TS 15066:2016~\cite{iso2016collaborative}. This barrier function ensures that the contact force between the $i$th soft robot part and the $j$th obstacle remains below the threshold.
\end{enumerate}

\subsection{Achieving Effective Motion Behavior via Higher-Order Control Lyapunov Functions}
With the barrier conditions established, the next step is to define motion objectives that encourage task completion. 
%
Although the \ac{HOCLF} framework affords substantial flexibility in defining task objectives, this letter concentrates on operational-space regulation; comparable objectives could also be formulated for trajectory tracking, configuration-space regulation and contact-force control.

\textbf{Operational Space Regulation \ac{HOCLF}.}
Let $\bm{p}_i(\bm{q}) \in \mathbb{R}^2$ be the Cartesian pose of the tip of the $i$th segment and $\bm{p}_{\mathrm{goal}} \in \mathbb{R}^2$ the desired target. Then the \ac{HOCLF} function
\begin{equation}
        V_{\mathrm{tsr},i}(\bm{x}) = \lVert \bm{p}_i(\bm{q}) - \bm{p}_i^\mathrm{d} \rVert_2^2
\end{equation}
encourages convergence of the tip of the $i$th segment toward the target position.

\section{Differentiable Polygon Distance Metric}

\begin{figure}
\vspace{+4mm}
  \centering
    \includegraphics[width=1.0\linewidth]{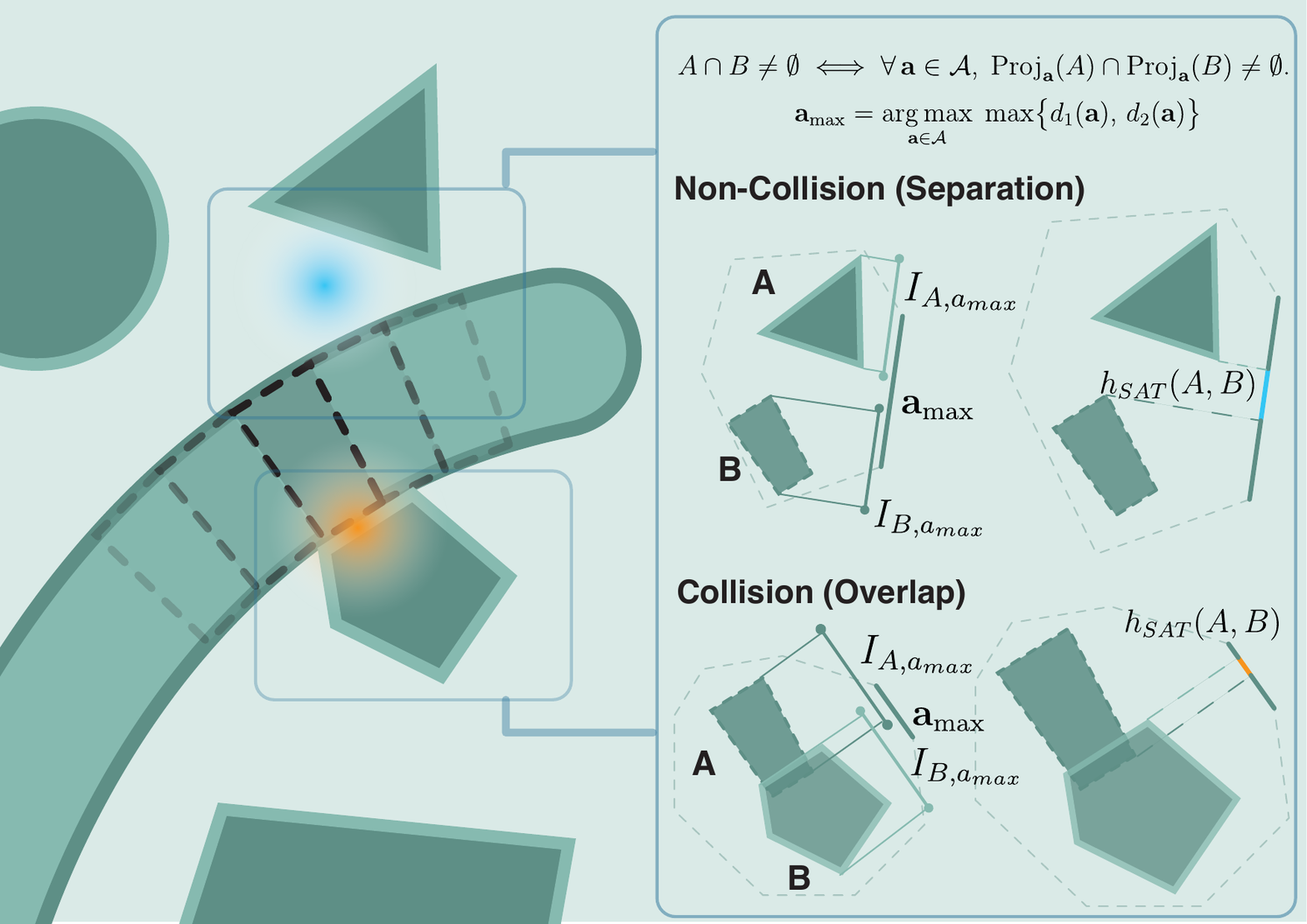}
  \caption{
    \textbf{Illustration of \ac{SAT} Polygon Distance Metrics.}
    Visualization of \ac{SAT}-based polygon distance metrics used for collision detection between the soft robot body and convex polygonal obstacles. Specifically, we illustrate the definition of the signed distance $h_{\mathrm{SAT}}(A, B)$ between convex polygons $A$ and $B$. 
    The right panel shows projection intervals $I_{A,a}$ and $I_{B,a}$ along the maximizing axis $a_{\max}$ in both separation (top) and overlap (bottom) scenarios. 
  }
  \label{fig:sat}
\end{figure} 
For \ac{HOCBF} constructions requiring an \(r\)th‐order continuously differentiable distance metric, the classical \ac{SAT}~\cite{dyn4jSAT}, despite its efficiency for convex polygons and widespread use over support mapping methods such as the \ac{GJK} algorithm~\cite{gjk1988}, is fundamentally unsuitable due to its reliance on non-smooth \(\min\) and \(\max\) operations. 
This lack of differentiability prevents its direct application in force- or distance-based high-order control formulations.
To address this limitation, prior work has introduced a differentiable variant of \ac{SAT}, named \ac{SSAT}, that approximates these non-differentiable operations using multi-level smoothing techniques to achieve \(C^\infty\) continuity~\cite{takasugi2024real}. However, the \ac{SSAT} overestimates the polygon separation distance, which can lead to a violation of the safety constraint in downstream applications. 
Alternatively, methods like DCOL~\cite{tracy2023differentiable} offer differentiable collision detection for convex primitives, but they cannot quantify penetration depth and are therefore unsuitable for collision resolution (e.g., computing the contact force).
Similarly, the randomized smoothing approach~\cite{montaut2023differentiable} and the accelerated optimization-based method~\cite{montaut2022collision} focus on differentiable collision checking and computational speed-up, respectively, but do not return differentiable signed distance or penetration depth, which limits their use in differentiable contact dynamics.
In this letter, we propose a new differentiable variant of \ac{SAT}, coined \ac{DCSAT}, which replaces the layered smoothing pipeline of \ac{SSAT}~\cite{takasugi2024real} with a single \ac{LSE} approximation while providing a conservation metric for calculating distance between convex polygons.
\ac{DCSAT} preserves the \(C^\infty\) differentiability required for high-order control while significantly simplifying implementation and reducing computational overhead.

\subsection{Definition of \ac{SAT}}
Before defining our distance metrics, we briefly revisit the \ac{SAT} framework, which forms the foundation of convex polygon–distance computations based on orthogonal projections.

\begin{definition}[Separating axis]
Let \(A, B \subset \mathbb{R}^d\) be convex sets.  
A unit vector \(\bm{a} \in \mathbb{R}^d\) is called a \emph{separating axis} for \(A\) and \(B\) if the projections of the sets onto \(\bm{a}\) are disjoint; that is, if either
\begin{equation*}
\max_{\bm{x} \in A} \bm{a}^\top \bm{x} < \min_{\bm{y} \in B} \bm{a}^\top \bm{y}
\quad \text{or} \quad
\max_{\bm{y} \in B} \bm{a}^\top \bm{y} < \min_{\bm{x} \in A} \bm{a}^\top \bm{x}.
\end{equation*}
\end{definition}

\begin{theorem}[Separating Axis Theorem~\cite{dyn4jSAT}]
\label{thm:SAT}
Two convex sets \(A\) and \(B\) in \(\mathbb{R}^d\) are disjoint if and only if there exists a separating axis between them.
\end{theorem}

\begin{lemma}[Sufficiency of Edge Normals in \(\mathbb{R}^2\)]
\label{lem:edge-normals}
Let \(A, B \subset \mathbb{R}^2\) be convex polygons. Then it suffices to test for separation along the set of directions orthogonal to the edges of \(A\) and \(B\). Specifically, let \(S(A)\) and \(S(B)\) denote the sets of edge directions of \(A\) and \(B\), respectively. Define $\mathcal{A} = \left\{ \bm{\ell}^\perp : \bm{\ell} \in S(A) \cup S(B) \right\}$,
where \(\bm{\ell}^\perp\) denotes the unit vector orthogonal to edge \(\bm{\ell}\). Then \(\mathcal{A}\) is a complete set of candidate separating axes.

\begin{proof}
By the Separating Axis Theorem (\cref{thm:SAT}), if \(A\) and \(B\) are disjoint, there exists a direction \(\mathbf{n} \in \mathbb{S}^1\) such that their projections onto \(\mathbf{n}\) do not overlap:
\begin{equation*}\small
\max_{x \in A} \mathbf{n}^\top x < \min_{y \in B} \mathbf{n}^\top y \quad \text{or} \quad \max_{y \in B} \mathbf{n}^\top y < \min_{x \in A} \mathbf{n}^\top x.
\end{equation*}
Since the support function of a convex polygon is piecewise linear and attains its extrema at vertices, any separating direction must be orthogonal to some edge of \(A\) or \(B\). Therefore, it suffices to test separation along directions in \(\mathcal{A}\). If no such direction yields separation, \(A\) and \(B\) must intersect.
\end{proof}
\end{lemma}

By combining \cref{thm:SAT} and \cref{lem:edge-normals}, separation testing—and more specifically, the computation of separating distances—can be reduced to a finite set of one-dimensional projections. These observations underlie the projection-based metrics (\ac{SAT}, \ac{SSAT}, \ac{DCSAT}) used throughout this letter.

Following by the above statements, we could derive the distance between two polygons A and B.

\textbf{Projections.}
Let $\mathbf{A}_i,i\in \mathcal{I}_A,  \mathbf{B}_j, i\in \mathcal{I}_B$ denote the vertices of convex sets $A, B$, respectively, where \(\mathcal{I}_A\) and \(\mathcal{I}_B\) denote the index sets of the vertices of $A, B$ respectively.
For each axis \( \mathbf{a} \in \mathcal{A} \), the scalar projections of polygon vertices \(\mathbf{A}_i\) for \(i \in \mathcal{I}_A\), and \(\mathbf{B}_j\) for \(j \in \mathcal{I}_B\), onto \(\mathbf{a}\) are defined as $A_{i,a} = \mathbf{a}^\top \mathbf{A}_i$ and $B_{j,a} = \mathbf{a}^\top \mathbf{B}_j$.

\textbf{Per-axis distance.}
The signed separation (positive if separated, negative if overlapping) between the projected intervals along axis \( \mathbf{a} \) is given by
\begin{equation}\label{eq:separation}\small
    g_{\mathrm{SAT}}(\mathbf{a}) = \max \bigl\{ d_1(\mathbf{a}),\, d_2(\mathbf{a}) \bigr\}.
\end{equation}
where
\begin{equation}
\footnotesize
\begin{split}
  d_1(\mathbf{a}) = \min_{j\in \mathcal{I}_B} B_{j,a} - \max_{i\in \mathcal{I}_A} A_{i,a}, \:
  d_2(\mathbf{a}) = \min_{i\in \mathcal{I}_A} A_{i,a} - \max_{j\in \mathcal{I}_B} B_{j,a},
\end{split}
\end{equation}

\textbf{Global metric.}
The overall \ac{SAT}-based signed distance between polygons \( A \) and \( B \) is then defined as
\begin{equation}\small
    h_{\mathrm{SAT}}(A,B)
  \;=\;
  \max_{\mathbf{a} \in \mathcal{A}}\, g_{\mathrm{SAT}}(\mathbf{a}).
\end{equation}
This value is positive when the polygons are separated, zero if they touch, and negative when they overlap.
\vspace{12pt}

\subsection{Differentiable Conservative SAT (\ac{DCSAT})}
\label{sec:onelog}


To enable safe and differentiable distance evaluation for use in \ac{HOCBF}+\ac{HOCLF} controllers, and resolve the underestimation of the polygon extents present in \ac{SSAT}~\cite{takasugi2024real}, we introduce the \ac{DCSAT} metric, which provides a conservative estimate of the polygon separation/penetration distance at an increased computational efficiency compared to \ac{SSAT}. Unlike \ac{SSAT}, which smooths each intermediate geometric operation and is primarily designed for quadrilateral shapes, \ac{DCSAT} operates directly on global signed separation distances. It applies a single \ac{LSE} operation to obtain a \( C^\infty \) approximation of the \ac{SAT} metric that naturally extends to arbitrary convex polygons. Furthermore, \ac{DCSAT} consistently underestimates the true separation distance, ensuring that safety constraints enforced via \acp{CBF} remain valid even under model uncertainty or near-contact conditions. This makes it particularly suitable for collision-aware control of systems with complex geometries.

\begin{definition}[Differentiable Conservative SAT (\ac{DCSAT})]
Let \(A, B \subset \mathbb{R}^2\) be convex polygons, and let \(\mathcal{A}\) denote the set of separating axes. For each axis \(\mathbf{a} \in \mathcal{A}\), define the separation terms \(d_1(\mathbf{a})\) and \(d_2(\mathbf{a})\) as in~\eqref{eq:separation}. Let
\begingroup
\small
\begin{equation*}
  \mathcal{D} = \left\{ d_m(\mathbf{a}) \;\middle|\; \mathbf{a} \in \mathcal{A},\; m \in \{1,2\} \right\}.
\end{equation*}
\endgroup
Then, the \ac{DCSAT} distance is defined as
\begingroup
\small
\begin{equation*}
  h_{\mathrm{DCSAT}}(A, B) :=  \frac{1}{\alpha_{\max}} \log \left(\sum_{d \in \mathcal{D}} e^{\alpha_{\max} d}\right) - \frac{\log(2|\mathcal{A}|)}{\alpha_{\max}}.
\end{equation*}
\endgroup
\end{definition}

\begin{lemma}[Bounds for \ac{DCSAT}]
\label{lemma:DCSAT}
\begingroup
\small
\begin{equation*}
  -\frac{\log(2|\mathcal{A}|)}{\alpha_{\max}} \le h_{\mathrm{DCSAT}}(A,B) - h_{\mathrm{SAT}}(A,B)
  \le 0.
\end{equation*}
\endgroup
\end{lemma}

\begin{proof}
Let \(M = \max \mathcal{D} = h_{\mathrm{SAT}}(A,B)\), and define $e := \frac{1}{\alpha_{\max}} \log\left( \sum_{d \in \mathcal{D}} e^{\alpha_{\max} d} \right)$.

\textbf{Lower bound:} One term in the sum is \(e^{\alpha_{\max} M}\), so

\begingroup
\small
\begin{equation*}
\begin{aligned}
  \sum e^{\alpha_{\max} d} \ge e^{\alpha_{\max} M}
  &\Rightarrow e \ge M
  \\\Rightarrow h_{\mathrm{DCSAT}}(A,B)& - M \ge -\frac{\log(2|\mathcal{A}|)}{\alpha_{\max}}.    
\end{aligned}
\end{equation*}
\endgroup

\textbf{Upper bound:} Since \(|\mathcal{D}| = 2|\mathcal{A}|\), we have

\begingroup
\small
\begin{equation*}
\begin{aligned}
  \sum e^{\alpha_{\max} d} \le 2|\mathcal{A}| \cdot e^{\alpha_{\max} M}
  \Rightarrow e \le M + \frac{\log(2|\mathcal{A}|)}{\alpha_{\max}}
  \\\Rightarrow h_{\mathrm{DCSAT}}(A,B) - M \le 0.    
\end{aligned}
\end{equation*}
\endgroup
\end{proof}

\begin{theorem}[Conservative Approximation via \ac{DCSAT}]
\label{lem:DCSAT_safety}
Let

\begingroup
\small
\begin{equation*}\small
\begin{aligned}
  e := \frac{1}{\alpha_{\max}} \log\left( \sum_{d \in \mathcal{D}} e^{\alpha_{\max} d} \right),
  h_{\mathrm{DCSAT}}(A, B) := e - \frac{\log(2|\mathcal{A}|)}{\alpha_{\max}}.
\end{aligned}
\end{equation*}
\endgroup

Then:
\begin{itemize}
  \item[(a)] \( h_{\mathrm{DCSAT}}(A, B) \in C^\infty \),
  \item[(b)] \( h_{\mathrm{DCSAT}}(A, B) \le h_{\mathrm{SAT}}(A, B) \) for all \(A, B\),
  \item[(c)] If \( h_{\mathrm{DCSAT}}(A, B) = 0 \), then the true SAT distance satisfies
  \begingroup
  \small
  \begin{equation}
    0 \le h_{\mathrm{SAT}}(A, B) < \frac{\log(2|\mathcal{A}|)}{\alpha_{\max}}.
  \end{equation}
  \endgroup
\end{itemize}
\end{theorem}

\begin{proof}
\textbf{(a)} Since \(e\) is composed of exponential and logarithmic operations over a finite sum of smooth terms, it follows that \(e \in C^\infty\). Subtracting a constant preserves smoothness, hence \(h_{\mathrm{DCSAT}} \in C^\infty\).
\textbf{(b)} Follows directly from~\cref{lemma:DCSAT}.
\textbf{(c)} If \(h_{\mathrm{DCSAT}}(A, B) = 0\), then by definition, $e = \frac{\log(2|\mathcal{A}|)}{\alpha_{\max}}$.

From~\cref{lemma:DCSAT}, it follows that
\[\small
    0 \le h_{\mathrm{SAT}}(A, B) \le \frac{\log(2|\mathcal{A}|)}{\alpha_{\max}},
\]
since \(h_{\mathrm{DCSAT}} \le h_{\mathrm{SAT}}\), with $h_{\mathrm{DCSAT}} =0$.

which confirms the conservative approximation of the real distance: if the smoothed metric \ac{DCSAT} hits zero, the true \ac{SAT} distance remains non-negative.
\end{proof}

The \ac{DCSAT} procedure is summarized in~\cref{alg:DCSAT}.
{\small
\begin{algorithm}[ht]
\caption{Differentiable Conservative SAT (DCSAT)}
\label{alg:DCSAT}
\begin{algorithmic}[1]
\Require Convex polygons \(A = \{\mathbf{A}_i\}\), \(B = \{\mathbf{B}_j\}\); sharpness \(\alpha > 0\)
\State \(\mathcal{A} \gets \{ \bm{\ell}^\perp : \bm{\ell} \in S(A) \cup S(B) \}\)
\ForAll{\(\mathbf{a} \in \mathcal{A}\)}
  \State \(d_1 \gets \min_j(\mathbf{a}^\top \mathbf{B}_j) - \max_i(\mathbf{a}^\top \mathbf{A}_i)\)
  \State \(d_2 \gets \min_i(\mathbf{a}^\top \mathbf{A}_i) - \max_j(\mathbf{a}^\top \mathbf{B}_j)\)
  \State Add \(d_1, d_2\) to \(\mathcal{D}\)
\EndFor
\State \(h \gets \frac{1}{\alpha} \log \sum_{d \in \mathcal{D}} e^{\alpha d} - \frac{\log(2|\mathcal{A}|)}{\alpha}\)
\State \Return \(h\)
\end{algorithmic}
\end{algorithm}
}

\begin{figure}[tb]
\vspace{-0mm}
\centering
{\small
\subfloat[\ac{SSAT}~\cite{takasugi2024real} vs. \ac{SAT}]{%
  \includegraphics[width=0.49\columnwidth]{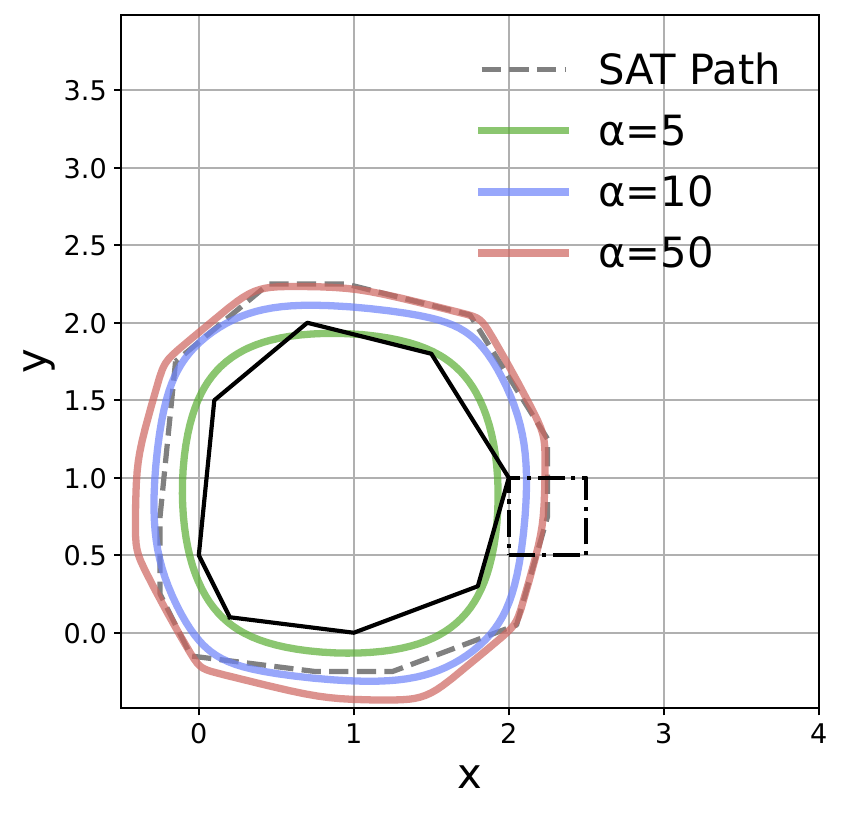}%
  \label{fig:ssat}
}
\subfloat[\ac{DCSAT} (ours) vs. \ac{SAT}]{%
  \includegraphics[width=0.49\columnwidth]{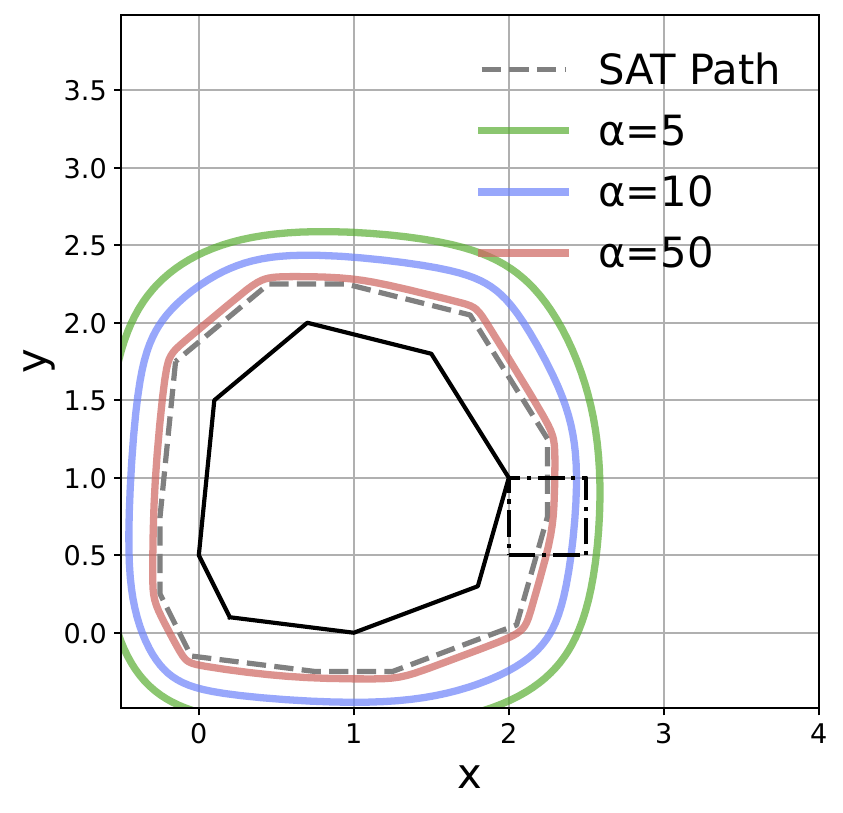}%
  \label{fig:dcsat}
}}
\vspace{-1mm}
\caption{
    \textbf{Qualitative Benchmarking of \ac{DCSAT}.}
    Comparison of zero-level contours between smoothed / differentiable polygon distance metrics and the classical \ac{SAT}. Specifically, we compare the position of the square polygon at a zero distance with the 8-sided polygon according to the respective distance metric. 
    In both cases, we report the resulting contours for various sharpness parameters~$\alpha$, and the gray dashed curve represents the true zero-level set of the classical, but not differentiable \ac{SAT} (i.e., the ground-truth), where the centroid of polygon~$B$ is in contact with polygon~$A$. 
}
\label{fig:DCSATcomparison}
\end{figure}

\subsection{Benchmarking DCSAT}
While the original \ac{SSAT} formulation~\cite{takasugi2024real} was designed specifically for rectangles or axis-aligned quadrilaterals, its underlying principle---approximating projection half-extents using smooth absolute value functions---can be generalized to arbitrary convex shapes. 
To enable a fair comparison across general geometries, we implemented this natural extension and evaluated it against our proposed \ac{DCSAT} and the classical \ac{SAT}~\cite{dyn4jSAT}.
\cref{fig:DCSATcomparison} visualizes the zero-level contours of both smooth metrics relative to the classical SAT boundary.
Notably, \ac{DCSAT} consistently generates a conservative underestimation of the true separation distance, as verified by the consistently negative values in the \ac{DCSAT} error column of \cref{tab:distance-bench}, ensuring that the smoothed constraint remains valid under model error or near-contact conditions.
In contrast, \ac{SSAT}~\cite{takasugi2024real} tends to overestimate the true separation at low sharpness levels, as can be seen in the positive values in the SSAT error column of \cref{tab:distance-bench}, which can compromise safety-critical guarantees in barrier-based control.

Although both smoothed variants are implemented efficiently in JAX, \ac{DCSAT} achieves comparable or better computational performance in large batch collision detections compared to \ac{SSAT}~\cite{takasugi2024real} despite its global formulation. 
As shown in \cref{tab:distance-bench}, it maintains fast execution across polygon sizes while preserving full \(C^\infty\) smoothness. To evaluate scalability, we benchmark the runtime of our method (DCSAT) against SSAT with increasing batch sizes (Table II). Each batch contains multiple robot segment–obstacle polygon pairs, where each obstacle polygon has $N = 8$. 
Results show that DCSAT consistently outperforms SSAT in runtime, and the relative speedup grows with the batch size, from 1.53× at batch size 32 to 2.39× at batch size 256.
Taken together, these results suggest that \ac{DCSAT} offers a robust and practical surrogate to classical \ac{SAT} for use in applications that require differentiability and a conservative approximation of the polygon separation distance.

\begin{table}[tb]
\centering
\scriptsize
\caption{
Average runtime (RT) in milliseconds over 1,000 trials for a batch of 32 four-sided robot segment–obstacle polygon pairs, where each obstacle polygon has $N$ sides. All polygons are aligned to be just in contact. Speedup is measured relative to SSAT. Distance errors are computed as the mean of the deviation from the \ac{SAT} metric normalized by the obstacle size. Additionally, we report the minimum and maximum observed across all pairs. 
}
\label{tab:distance-bench}
\setlength{\tabcolsep}{0.8pt}
\renewcommand{\arraystretch}{1.0}
\begin{tiny}
\begin{tabular}{@{}rcccccc@{}}
\toprule
  $N$ 
  & \textbf{\ac{SSAT} RT} $\downarrow$
  & \textbf{\ac{DCSAT} RT} $\downarrow$
  & \textbf{\ac{SAT} RT (ours)} $\downarrow$
  & \textbf{Speedup} $\uparrow$ 
  & \textbf{\ac{SSAT} Err. [min, max]} 
  & \textbf{\ac{DCSAT} Err. [min, max]} \\
\midrule
4   & 0.046~\si{ms} & 0.031~\si{ms} & 0.025~\si{ms} & 1.51 & -12.0\% [-37.6\%, -0.1\%] & -2.8\% [-3.5\%, -1.9\%] \\
8   & 0.064~\si{ms} & 0.041~\si{ms} & 0.033~\si{ms} & 1.56 & -1.1\% [-21.9\%, +1.2\%] & -2.6\% [-3.2\%, -1.8\%] \\
16  & 0.096~\si{ms} & 0.053~\si{ms} & 0.039~\si{ms} & 1.81 & +0.2\% [-2.3\%, +1.2\%]  & -2.9\% [-3.4\%, -2.0\%] \\
32  & 0.129~\si{ms} & 0.064~\si{ms} & 0.046~\si{ms} & 2.00 & +0.4\% [-2.0\%, +1.0\%]  & -2.8\% [-3.3\%, -2.1\%] \\
64  & 0.165~\si{ms} & 0.076~\si{ms} & 0.063~\si{ms} & 2.17 & +0.5\% [-0.4\%, +1.1\%]  & -2.8\% [-3.4\%, -2.3\%] \\

\bottomrule
\end{tabular}
\end{tiny}
\end{table}

\begin{table}[tb]
\centering
\caption{Average runtime per segment–polygon pair (ms) with increasing batch size, the polygons have eight sides. Each value is averaged over 1000 trials. Speedup is measured relative to SSAT.}
\label{tab:batch-size-bench}
\setlength{\tabcolsep}{1.5pt}   
\begin{scriptsize}
\begin{tabular}{@{}rccc@{}}
\toprule
\textbf{Batch Size} 
& \textbf{\ac{SSAT} (ms)} $\downarrow$
& \textbf{\ac{DCSAT} (ours) (ms)} $\downarrow$
& \textbf{Speedup} $\uparrow$ \\
\midrule
  $32$  & $0.0708$ & $0.0463$ & $1.53$ \\
  $64$  & $0.0946$ & $0.0514$ & $1.84$ \\
 $128$  & $0.1330$ & $0.0653$ & $2.04$ \\
 $256$  & $0.2356$ & $0.0985$ & $2.39$ \\
\bottomrule
\end{tabular}
\end{scriptsize}
\end{table}

\section{Experiments}

\subsection{Baseline Methods}
\paragraph{Safety Unaware HOCLF.}
This baseline captures the typical approach found in current soft-robotic control research \cite{della2023model}: controlling the system without explicitly enforcing safety constraints. Here, we drop the HOCBF terms from the QP and retain only the HOCLF component—specifically, the Operational-Space Regulation HOCLF—within the QP constraints.

\paragraph{Contact Avoidance Artificial Potential Field.}
The purpose of this baseline is the represent the scenario that is widespread in the rigid robotics literature (e.g., collision avoidance~\cite{haddadin2013towards}) - the aim to fully avoid contact - in this case via application of a repulsive \ac{APF} approach~\cite{khatib1986potential}.
We adopt an integral-free (i.e., $K_\mathrm{i} = 0$) variant of operational-space impedance controller from \eqref{eq:model_based_operational_space_controller} to operate on positions of the soft robot segment tips with $\bm{\omicron} = \big[\bm{p}_1^\top, \dots, \bm{p}_N^\top\big]^\top \in \mathbb{R}^{2N}$. Combined with the repulsive artificial potential force $\bm{f}_\mathrm{rep}$, the resulting control law is:
\begin{equation}\footnotesize
\begin{split}
    \bm{u} =& \: \bm{u}_\omicron + \sum_{i=1}^{N_\mathrm{srpoly}} \sum_{j=1}^{N_\mathrm{obs}} \bm{J}_{\mathrm{c}_{i,j}}^\top \bm{f}_{\mathrm{rep}_{i,j}},
    \\
    \bm{f}_{\mathrm{rep}_{i,j}} =& \: \begin{cases}
		k_\mathrm{rep} \frac{h_{i,j}^{-1} - r_\mathrm{safe}^{-1}}{h_{i,j}^2} \,\bm{n}_{\mathrm{c}_{i,j}}, & \text{if $h_{i,j} \leq r_\mathrm{safe}$}\\
        \bm{0}_2, & \text{if $h_{i,j} > r_\mathrm{safe}$}
     \end{cases}.
\end{split}
\end{equation}
Please note that the proportional term of $\bm{u}_\omicron$ in \eqref{eq:model_based_operational_space_controller} corresponds to force stemming from an attractive \ac{APF}.

\paragraph{Contact Avoidance HOCBF+HOCLF.}
The purpose of this baseline is similar to the last one, but instead of relying on an artificial potential field to avoid contact and a classic operational space controller for incorporating the task objective, we rely here on the HOCBF+HOCLF framework by combining a \emph{Contact Avoidance HOCBF} with a \emph{Operational Space Regulation HOCLF}.

\paragraph{Contact Force-Limit HOCBF Filter.}
In this baseline, we replace the \ac{HOCLF} objectives with a nominal operational-space controller whose commands are filtered for safety by solving the \ac{QP} with \emph{Contact Force Limit HOCBF} constraints. Specifically, we employ $\mathbf{u}_{\omicron}(\mathbf{x})$ from \eqref{eq:model_based_operational_space_controller}, where the operational-space is defined as $\bm{\omicron} = \big[\bm{p}_1^\top, \dots, \bm{p}_N^\top\big]^\top \in \mathbb{R}^{2N}$, representing the positions of the segment tips.


\subsection{Implementation \& Simulation Details}
We build on the \texttt{CBFpy}~\cite{morton2025oscbf} package that offers an easy-to-use and high-performance implementation of (high-order) \ac{CBF}+\ac{CLF} in JAX while leveraging analytical gradients obtained via autodifferentiation.
Our simulations consider a planar, two-segment PCS soft robot ($N = 2$) implemented in the \texttt{JSRM} package~\cite{stolzle2024experimental}. Each segment is \SI{0.13}{m} long with a backbone radius of \SI{0.02}{m}; the elastic modulus is \SI{2}{kPa} and the shear modulus is \SI{1}{kPa}. The actuation matrix is set to the identity, $\bm{A}(\bm{q}) = \mathbb{I}_{3N}$. No state or input bounds are enforced. The closed-loop ODE is integrated with a numerical solver implementing Tsitouras' 5/4 method.

For the contact model, we choose $k_\mathrm{c} = \SI{3000}{N \per m}, c_\mathrm{c} = 0$ and set the contact force limit to $F_{\mathrm{c},\max,j} = \SI{15}{N}$ and, in contact-avoidance scenarios, impose a minimum safety clearance of $r_\mathrm{safe} = \SI{0.005}{m}$, and a repulsive artificial potential field stiffness of $k_\mathrm{rep} = \SI{3000}{N/M}$. The $\varepsilon$ value for softplus-based smoothing force is \num{2e-4}.

\subsection{Navigation Results}

We implemented multiple scenarios to validate our framework. Among them, the search \& rescue task is the most representative for comparing different baselines and is therefore highlighted in this letter (see supplementary materials for additional scenarios).
%
%
%
Specifically, we compare the behavior of the \emph{Contact Force-Limit HOCBF+HOCLF} control strategy against the baselines.
We present the results in \cref{fig:search_rescue_sequence_of_stills} and \cref{fig:search_rescue_time_plot} as sequences of stills and time evolution plots, respectively.
Indeed, the results demonstrate that the \emph{safety-unaware} controller (1) generates high contact forces that are potentially unsafe, (2) the \emph{contact-avoidance} baselines is not able to complete the task as it exhibits an overly conservative behavior and the soft robot is not able to exploit its embodied intelligence, (3) a model-based operational-space regulator with \ac{HOCBF}-based safety filter is not able to complete the task as the control strategy is not able to resolve the conflicts between task objective and safety constraints, and (4) the proposed \emph{contact-force limit} exhibits very good task performance while ensuring safety by restricting the maximum contact force.

\begin{figure}[htb]
  \vspace{-4mm}
  \centering
  \includegraphics[width=1.0\linewidth, trim={5 5 5 5}]{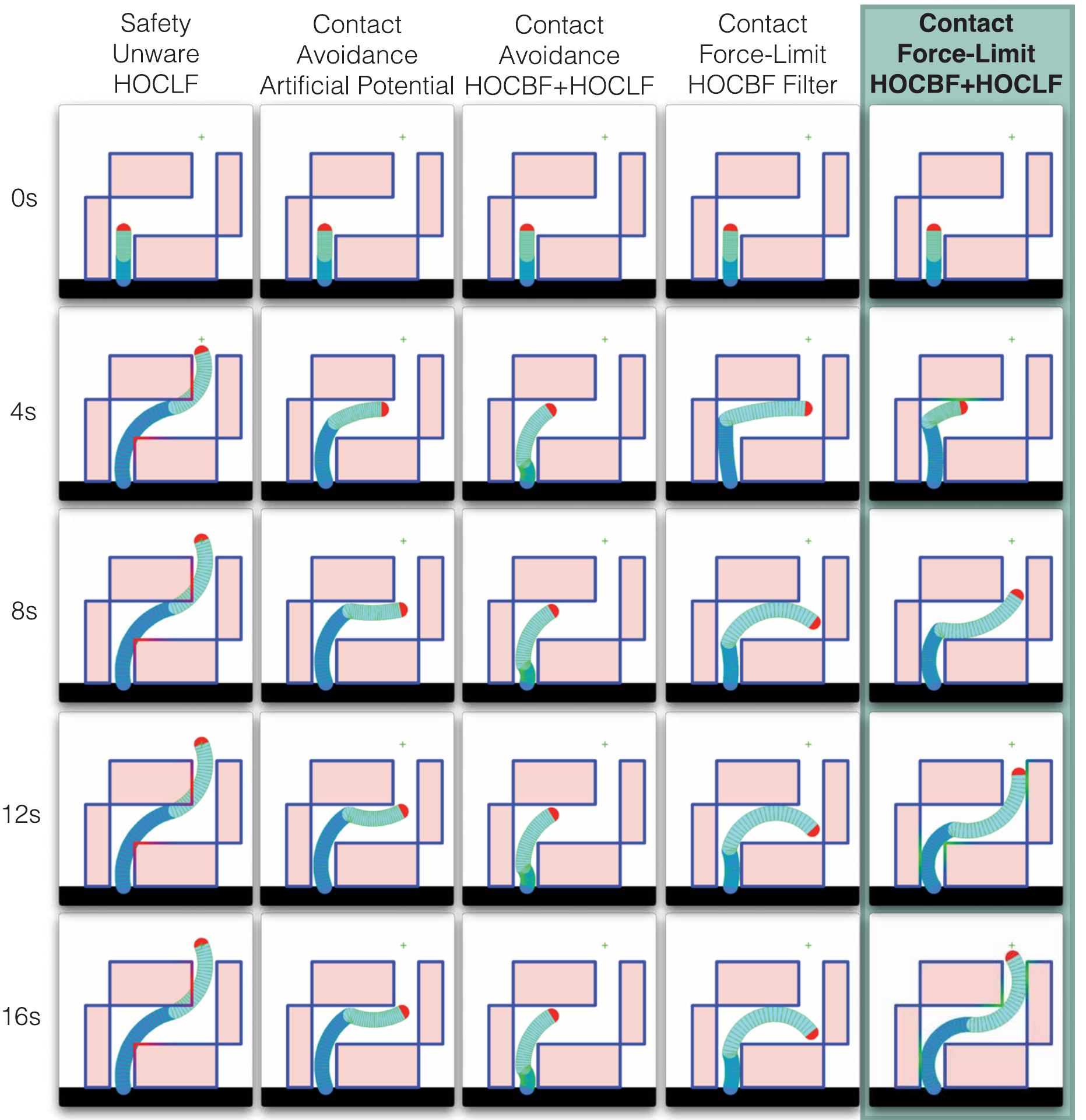}
  \vspace{-1mm}
  \caption{
    \textbf{Search \& Rescue Sequence of Stills.}
    Sequence of stills for the system evolution in the \emph{search \& rescue} scenario for five different control paradigms. 
    Contact interaction states are visualized with color cues: blue indicates no contact, green denotes safe contact, and red highlights contact above the maximum allowable force. The green cross denotes the task-space goal for this task.
    \emph{Safety Unaware HOCLF:} Only optimizing for the task objective encoded in the \ac{HOCLF} and disregarding safety constraints.
        \emph{Contact Avoidance Artificial Potential:}  Fully avoiding contact, which is the common paradigm in rigid robotics, using an artificial potential field~\cite{khatib1986potential},
    \emph{Contact Avoidance HOCBF+HOCLF:} Fully avoiding contact using the \ac{HOCBF}+\ac{HOCLF} framework.
    \emph{Contact Force-Limit HOCBF Filter:} Embracing contact with the environment while ensuring that the contact forces remain within safe bounds via an HOCBF-based safety filter applied to a nominal operational-space controller~\cite{della2020model, stolzle2024guiding, stolzle2025phdthesis}.
    \emph{Contact Force-Limit HOCBF+HOCLF (proposed):} Also ensuring that the contact forces remain within safe bounds, but instead of a safety filter, directly solving a \ac{QP} with \ac{HOCBF} and \ac{HOCLF} constraints.
    \vspace{-1mm}
  }
  \label{fig:search_rescue_sequence_of_stills}
\end{figure}



\begin{figure}[htbp]
\vspace{-9mm}
  \centering
  \subfloat[Contact force evolution]{%
    \includegraphics[width=0.48\linewidth, trim={5 5 5 5}]{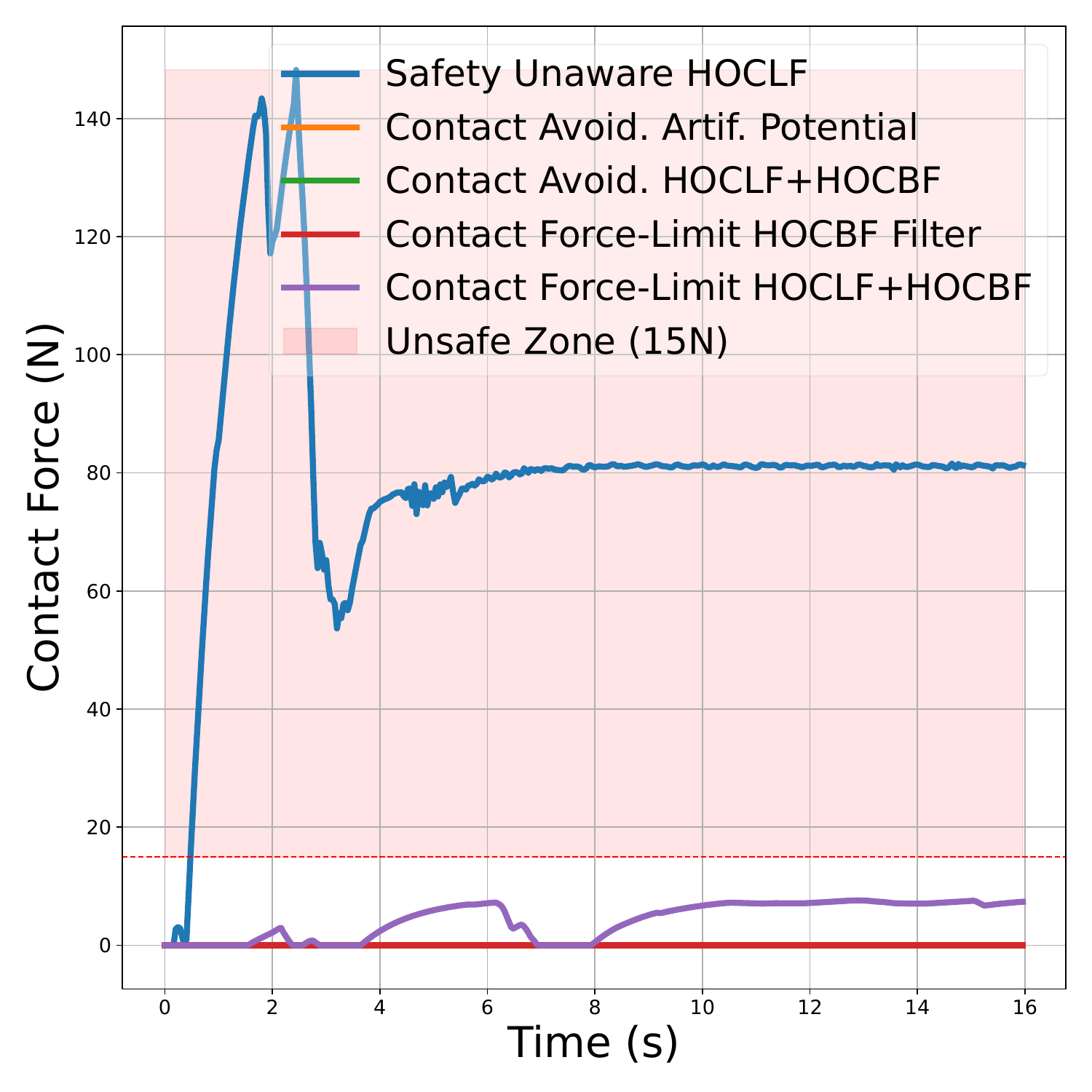}
  }
  \hfill
  \subfloat[Regulation error evolution]{%
    \includegraphics[width=0.48\linewidth, trim={5 5 5 5}]{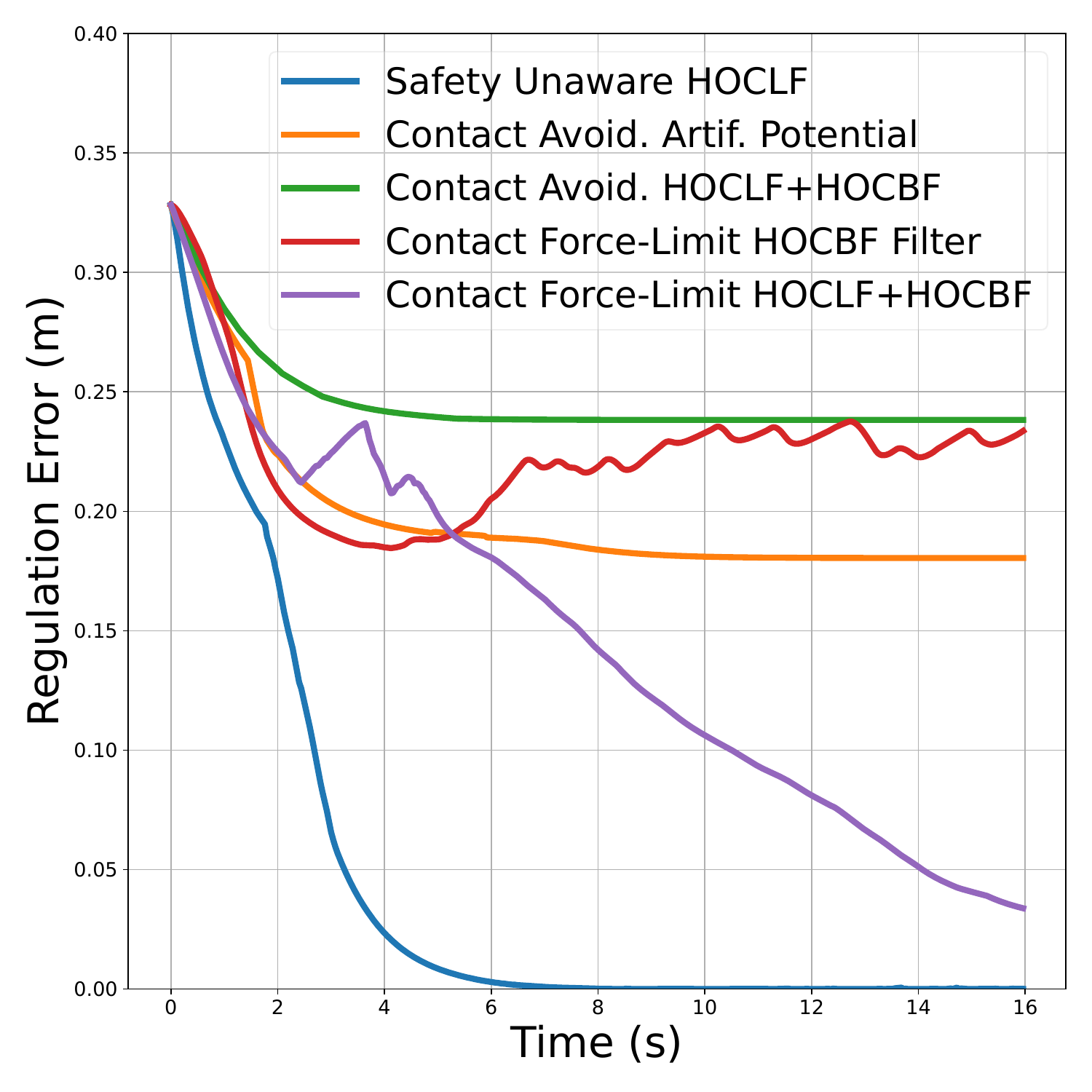}
  }

  \caption{\small
  \textbf{Search \& Rescue Time Evolution.}
  The contact force value, regulation error value across the robot for different controllers during a \emph{search \& rescue} scenario. Refer to the caption of \cref{fig:search_rescue_sequence_of_stills} for more details about the experiment and the considered controllers.
  }
  \label{fig:search_rescue_time_plot}
  \vspace{-1mm}
\end{figure}

\section{Conclusion}
This letter introduces a \acp{HOCBF}+\acp{HOCLF} framework~\cite{xiao2021highclbf} for the control of soft robots, grounded in a differentiable \ac{PCS} model~\cite{renda2018discrete, stolzle2024experimental} and a novel differentiable and conservative polygon distance metric \ac{DCSAT}.
Our method allows (1) soft robots to embrace contact with the environment but ensure safety via contact force-limit \acp{HOCBF} evaluated along the entire soft robot body, (2) the flexible and expressive specification of control objectives, such as shape and end-effector regulation or object manipulation, via \acp{HOCLF}.
The \ac{DCSAT} metric offers a~$C^\infty$ approximation of the classical \ac{SAT}, yielding conservative signed distances with reduced computational overhead compared to existing baseline methods that overestimate the distance, leading to potentially unsafe behavior.
Simulated experiments in navigation scenarios validate the framework's ability to maintain geometric safety and guide soft robots toward task objectives.
Current limitations of the work include the focus on planar settings, the simplistic contact model, and the reliance on accurate proprioception and exteroception.

\bibliographystyle{IEEEtran}
\bibliography{main.bib}

\begin{thebibliography}{10}
\providecommand{\url}[1]{#1}
\csname url@samestyle\endcsname
\providecommand{\newblock}{\relax}
\providecommand{\bibinfo}[2]{#2}
\providecommand{\BIBentrySTDinterwordspacing}{\spaceskip=0pt\relax}
\providecommand{\BIBentryALTinterwordstretchfactor}{4}
\providecommand{\BIBentryALTinterwordspacing}{\spaceskip=\fontdimen2\font plus
\BIBentryALTinterwordstretchfactor\fontdimen3\font minus \fontdimen4\font\relax}
\providecommand{\BIBforeignlanguage}[2]{{%
\expandafter\ifx\csname l@#1\endcsname\relax
\typeout{** WARNING: IEEEtran.bst: No hyphenation pattern has been}%
\typeout{** loaded for the language `#1'. Using the pattern for}%
\typeout{** the default language instead.}%
\else
\language=\csname l@#1\endcsname
\fi
#2}}
\providecommand{\BIBdecl}{\relax}
\BIBdecl

\bibitem{hall2019acceptance}
A.~K. Hall, U.~Backonja \emph{et~al.}, ``Acceptance and perceived usefulness of robots to assist with activities of daily living and healthcare tasks,'' \emph{Assistive Technology}, 2019.

\bibitem{haddadin2013towards}
S.~Haddadin, \emph{Towards safe robots: approaching Asimov’s 1st law}.\hskip 1em plus 0.5em minus 0.4em\relax Heidelberg: Springer Berlin, 2013, vol.~90.

\bibitem{khatib1987unified}
O.~Khatib, ``A unified approach for motion and force control of robot manipulators: The operational space formulation,'' \emph{IEEE Journal on Robotics and Automation}, vol.~3, no.~1, pp. 43--53, 1987.

\bibitem{pupa2024efficient}
A.~Pupa and C.~Secchi, ``Efficient iso/ts 15066 compliance through model predictive control,'' in \emph{2024 IEEE International Conference on Robotics and Automation (ICRA)}.\hskip 1em plus 0.5em minus 0.4em\relax IEEE, 2024, pp. 17\,358--17\,364.

\bibitem{ames2016control}
A.~D. Ames, X.~Xu \emph{et~al.}, ``Control barrier function based quadratic programs for safety critical systems,'' \emph{IEEE Transactions on Automatic Control}, vol.~62, no.~8, pp. 3861--3876, 2016.

\bibitem{ferraguti2020control}
F.~Ferraguti, M.~Bertuletti \emph{et~al.}, ``A control barrier function approach for maximizing performance while fulfilling to iso/ts 15066 regulations,'' \emph{IEEE Robotics and Automation Letters}, vol.~5, no.~4, pp. 5921--5928, 2020.

\bibitem{han2021reinforcement}
M.~Han, Y.~Tian \emph{et~al.}, ``Reinforcement learning control of constrained dynamic systems with uniformly ultimate boundedness stability guarantee,'' \emph{Automatica}, vol. 129, p. 109689, 2021.

\bibitem{rus2015design}
D.~Rus and M.~T. Tolley, ``Design, fabrication and control of soft robots,'' \emph{Nature}, vol. 521, no. 7553, pp. 467--475, 2015.

\bibitem{stolzle2025soft}
M.~St{\"o}lzle, N.~Pagliarani \emph{et~al.}, ``Soft yet effective robots via holistic co-design,'' \emph{arXiv preprint arXiv:2505.03761}, 2025.

\bibitem{dickson2025safe}
A.~K. Dickson, J.~C.~P. Garcia \emph{et~al.}, ``Safe autonomous environmental contact for soft robots using control barrier functions,'' \emph{IEEE Robotics and Automation Letters}, vol.~10, no.~11, pp. 11\,283--11\,290, 2025.

\bibitem{haggerty2023control}
D.~A. Haggerty, M.~J. Banks \emph{et~al.}, ``Control of soft robots with inertial dynamics,'' \emph{Science robotics}, vol.~8, no.~81, p. eadd6864, 2023.

\bibitem{patterson2024safe}
Z.~J. Patterson, W.~Xiao \emph{et~al.}, ``Safe control for soft-rigid robots with self-contact using control barrier functions,'' in \emph{2024 IEEE 7th International Conference on Soft Robotics (RoboSoft)}.\hskip 1em plus 0.5em minus 0.4em\relax IEEE, 2024, pp. 151--156.

\bibitem{rao2024towards}
P.~Rao, O.~Salzman, and J.~Burgner-Kahrs, ``Towards contact-aided motion planning for tendon-driven continuum robots,'' \emph{IEEE Robotics and Automation Letters}, 2024.

\bibitem{xu2024hybrid}
F.~Xu, X.~Kang, and H.~Wang, ``Hybrid visual servoing control of a soft robot with compliant obstacle avoidance,'' \emph{IEEE/ASME Transactions on Mechatronics}, vol.~29, no.~6, pp. 4446--4455, 2024.

\bibitem{iso2016collaborative}
I.~Standard, ``Iso/ts 15066: 2016: Robots and robotic devices--collaborative robots,'' \emph{International Organization for Standardization: Geneva, Switzerland}, 2016.

\bibitem{mengaldo2022concise}
G.~Mengaldo, F.~Renda \emph{et~al.}, ``A concise guide to modelling the physics of embodied intelligence in soft robotics,'' \emph{Nature Reviews Physics}, vol.~4, no.~9, pp. 595--610, 2022.

\bibitem{della2020model}
C.~Della~Santina, R.~K. Katzschmann \emph{et~al.}, ``Model-based dynamic feedback control of a planar soft robot: trajectory tracking and interaction with the environment,'' \emph{The International Journal of Robotics Research}, vol.~39, no.~4, pp. 490--513, 2020.

\bibitem{xiao2021highclbf}
W.~Xiao, C.~A. Belta, and C.~G. Cassandras, ``High order control lyapunov-barrier functions for temporal logic specifications,'' in \emph{2021 American Control Conference (ACC)}.\hskip 1em plus 0.5em minus 0.4em\relax IEEE, 2021, pp. 4886--4891.

\bibitem{renda2018discrete}
F.~Renda, F.~Boyer \emph{et~al.}, ``Discrete cosserat approach for multisection soft manipulator dynamics,'' \emph{IEEE Transactions on Robotics}, vol.~34, no.~6, pp. 1518--1533, 2018.

\bibitem{stolzle2024experimental}
M.~St{\"o}lzle, D.~Rus, and C.~Della~Santina, ``An experimental study of model-based control for planar handed shearing auxetics robots,'' in \emph{Experimental Robotics}.\hskip 1em plus 0.5em minus 0.4em\relax Cham: Springer Nature Switzerland, 2024, pp. 153--167.

\bibitem{takasugi2024real}
N.~Takasugi, M.~Kinoshita \emph{et~al.}, ``Real-time perceptive motion control using control barrier functions with analytical smoothing for six-wheeled-telescopic-legged robot tachyon 3,'' in \emph{2024 IEEE/RSJ International Conference on Intelligent Robots and Systems (IROS)}.\hskip 1em plus 0.5em minus 0.4em\relax IEEE, 2024, pp. 6802--6809.

\bibitem{della2023model}
C.~Della~Santina, C.~Duriez, and D.~Rus, ``Model-based control of soft robots: A survey of the state of the art and open challenges,'' \emph{IEEE Control Systems Magazine}, vol.~43, no.~3, pp. 30--65, 2023.

\bibitem{stolzle2024guiding}
M.~St{\"o}lzle, S.~S. Baberwal \emph{et~al.}, ``Guiding soft robots with motor-imagery brain signals and impedance control,'' in \emph{2024 IEEE 7th International Conference on Soft Robotics (RoboSoft)}.\hskip 1em plus 0.5em minus 0.4em\relax IEEE, 2024, pp. 276--283.

\bibitem{stolzle2025phdthesis}
M.~St{\"o}lzle, ``\BIBforeignlanguage{English}{Safe yet precise soft robots: Incorporating physics into learned models for control},'' Dissertation (TU Delft), Mechanical Engineering, Delft University of Technology, 09 2025.

\bibitem{dyn4jSAT}
W.~Bittle, ``Sat (separating axis theorem),'' \url{https://dyn4j.org/2010/01/sat/}, 2010, accessed: 2025-04-29.

\bibitem{gjk1988}
E.~G. Gilbert, D.~W. Johnson, and S.~S. Keerthi, ``A fast procedure for computing the distance between complex objects in three-dimensional space,'' \emph{IEEE Journal of Robotics and Automation}, vol.~4, no.~2, pp. 193--203, 1988.

\bibitem{tracy2023differentiable}
K.~Tracy, T.~A. Howell, and Z.~Manchester, ``Differentiable collision detection for a set of convex primitives,'' in \emph{2023 IEEE International Conference on Robotics and Automation (ICRA)}.\hskip 1em plus 0.5em minus 0.4em\relax IEEE, 2023, pp. 3663--3670.

\bibitem{montaut2023differentiable}
L.~Montaut, Q.~Le~Lidec \emph{et~al.}, ``Differentiable collision detection: a randomized smoothing approach,'' in \emph{2023 IEEE International Conference on Robotics and Automation (ICRA)}.\hskip 1em plus 0.5em minus 0.4em\relax IEEE, 2023, pp. 3240--3246.

\bibitem{montaut2022collision}
L.~Montaut, Q.~L. Lidec \emph{et~al.}, ``Collision detection accelerated: An optimization perspective,'' in \emph{RSS 2022-Robotics: Science and Systems}, 2022.

\bibitem{khatib1986potential}
O.~Khatib, ``The potential field approach and operational space formulation in robot control,'' in \emph{Adaptive and Learning Systems: Theory and Applications}.\hskip 1em plus 0.5em minus 0.4em\relax Springer, 1986, pp. 367--377.

\bibitem{morton2025oscbf}
D.~Morton and M.~Pavone, ``Safe, task-consistent manipulation with operational space control barrier functions,'' \emph{arXiv preprint arXiv:2503.06736}, 2025, accepted at IEEE/RSJ International Conference on Intelligent Robots and Systems (IROS), Hangzhou, 2025.

\end{thebibliography}

\vfill

\end{document}